\documentclass[letterpaper, 10 pt, conference]{ieeeconf}  
\usepackage{amsmath}
\usepackage{amssymb} 
\usepackage{algorithm} 
\usepackage{algpseudocode}
\usepackage{multirow}
\usepackage{booktabs,subcaption}
\usepackage{graphicx}
\usepackage{pgffor}
\usepackage{placeins}
\usepackage{makecell}
\usepackage{adjustbox}

\newtheorem{theorem}{Theorem}[section]

\usepackage[T1]{fontenc}
\usepackage{cite}
\usepackage{hyperref}
\newcommand{\argmax}{\mathop{\mathrm{arg\,max}}}

\IEEEoverridecommandlockouts                              

\title{\LARGE \bf
Viewpoint-Agnostic Manipulation Policies \\with Strategic \emph{Vantage} Selection 
}

\author{Sreevishakh Vasudevan,$^1$ Som Sagar,$^1$ and Ransalu Senanayake$^1$ 
\thanks{$^1$Laboratory for Learning Evaluation and Naturalization of Systems (LENS Lab), Arizona State University (ASU), USA. Emails: {\tt\small <svasud23,ssagar6,ransalu>@asu.edu}}}

\begin{document}

\maketitle
\thispagestyle{empty}
\pagestyle{empty}

\begin{abstract}

Since vision-based manipulation policies are typically trained from data gathered from a single viewpoint, their performance drops when the view changes during deployment. Naively aggregating demonstrations from numerous random views is not only costly but also known to destabilize learning, as excessive visual diversity acts as noise. We present \emph{Vantage}, a viewpoint selection framework to fine-tune any pre-trained policy on a small, strategically chosen set of camera poses to induce viewpoint-agnostic behavior. Instead of relying on costly brute-force search over viewpoints, \emph{Vantage} formulates camera placement as an information gain optimization problem in a continuous space. This approach balances exploration of novel poses with exploitation of promising ones, while also providing theoretical guarantees about convergence and robustness. Across manipulation tasks and policy families, \emph{Vantage} consistently improves success under viewpoint shifts compared to fixed, grid, or random data selection strategies with only a handful of fine-tuning steps. Experiments conducted on simulated and real-world setups show that \emph{Vantage} increases the task success rate by $\approx$25\% for diffusion policies, and yields robust gains in dynamic-camera settings. GitHub:  \url{https://github.com/sreevishakhv/Vantage_Public}

\end{abstract}

\section{INTRODUCTION}
\label{sec:intro}

\begin{figure}[!t]
    \centering
    \begin{subfigure}[t]{\linewidth}
        \centering
        \includegraphics[width=\linewidth]{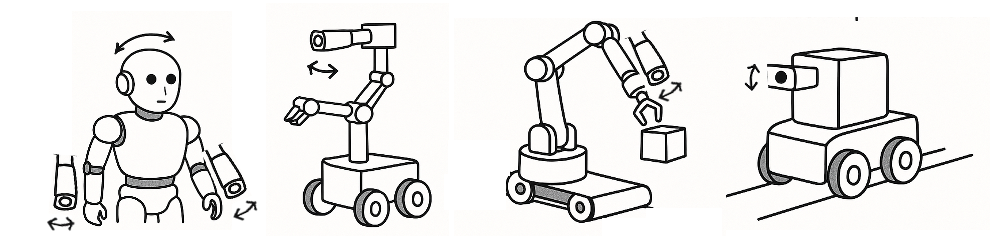}
        \caption{}
        \label{fig:viewpoint_a}
    \end{subfigure}
    \vspace{1em}
    \begin{subfigure}[t]{\linewidth}
        \centering
        \includegraphics[width=\linewidth]{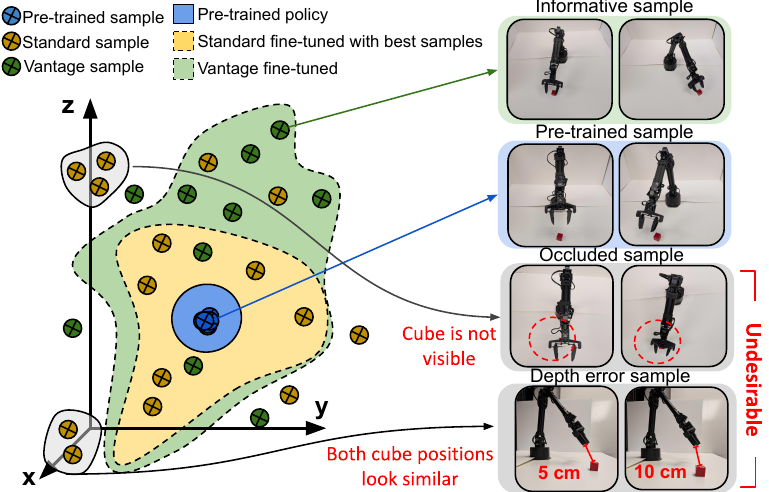}
        \caption{}
        \label{fig:viewpoint_b}
    \end{subfigure}
    \caption{
    (a) In contemporary and future robotic platforms, cameras are frequently mounted on moving bodies or joints, causing dynamic changes in viewpoints. 
     (b) Illustration of selecting camera viewpoints when training manipulation policies so that they become agnostic to camera location at test time. Randomly selecting viewpoints deteriorates performance. Consider a setup where a camera can be placed anywhere in the $x\!-\!y\!-\!z$ space to observe the manipulator and its surrounding. Colored regions indicate where the camera can be placed for each policy for good performance. Pre-trained policies work only when a camera is placed in the narrow region (blue) of high accuracy, where it was originally trained. Standard fine-tuning (yellow), which relies on samples (i.e., collecting demonstrations) collected from randomly or uniformly placed camera viewpoints, spreads demonstration collection and fine-tuning budget across many \emph{uninformative} regions. Such samples with occlusion or errors in depth perception (gray) can even hinder learning performance. \emph{Vantage}, in contrast, strategically selects a small number of informative viewpoints (green), targeting areas that maximize downstream task performance. This allows vantage-fine-tuned policies to perform well even in dynamic camera settings, described in (a) and Section~\ref{sec:intro}.}
    \label{fig:viewpoint}
\end{figure}

Modern robot manipulation policies trained with visual inputs have achieved levels of precision and adaptability that were once considered far-fetched. Yet, behind this success lies a subtle but critical limitation: most policies learn to act correctly only from the camera viewpoint they were trained on. Even slight changes in viewpoints can reduce a confident policy to a failure. This brittleness arises because camera viewpoint has a decisive influence on what a robot sees. For instance, an overhead view may offer an unobstructed global view of a workspace but risks being occluded as the arm moves. Similarly, a side-mounted camera tracks motion well but distorts object geometry, making attributes like orientation of objects harder to infer. More critically, in dynamic camera settings such as in humanoid robots with constantly moving heads, mobile manipulators with the camera mounted on pan-tilt heads or mobile bases, robots on moving assembly lines, viewpoint-specialized policies quickly lose their reliability, restricting robots from performing robustly outside carefully staged environments~\cite{lee2022uncertainty,zhang2023affordance}.

One might expect that training on data from many random viewpoints would solve the problem. In practice, however, excessive variation from uninformative perspectives such as oblique, occluded, and overhead views act as corrupt data that hinders the model from learning stable, viewpoint-invariant representations. This phenomenon parallels the challenges of domain generalization, where exposure to overly diverse training distributions can hinder rather than help performance~\cite{zhang2017understanding,muandet2013domain}. Recent benchmarks also highlight how varying viewpoints affect models' ability to understand properties, affordances, and constraints in the workspace~\cite{atharva25}. These challenges warrant new methods that strategically select informative viewpoints while avoiding those that hinder policy learning. Since the landscape of robot learning is shifting toward fine-tuning generalist policies~\cite{kim2024openvla} for specific applications, it becomes essential to perform this strategic selection during fine-tuning. Further, from a machine learning perspective, fine-tuning on a good enough model is more stable and data-efficient than learning to be invariant from scratch. In this regime, the key question is no longer ``How can we learn from every possible viewpoint?'' but rather ``Which viewpoints are most valuable for fine-tuning?''

\begin{figure*}[t]
  \centering
  \includegraphics[width=\textwidth]{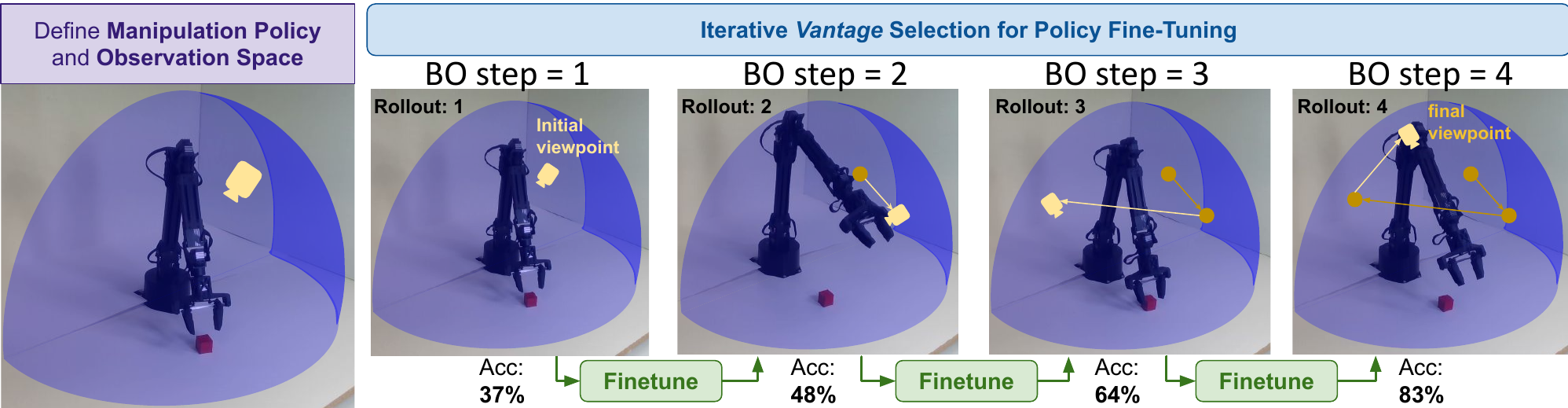}
    \caption{Overview of the \emph{Vantage} framework. Starting from a pre-trained manipulation policy and an initial camera viewpoint, the system iteratively selects additional viewpoints for fine-tuning. After each Bayesian optimization step (BO step), the updated policy is evaluated across the observation space, and the performance signal is used to guide the next selection. By progressively incorporating strategically chosen views, the policy becomes increasingly robust to viewpoint shifts, converging to a final model that generalizes across diverse views.}
  \label{fig:intro}
\end{figure*}

We introduce Vantage, a framework that iteratively suggests camera poses for targeted data collection, enabling strategic fine-tuning that improves policy success rate across diverse camera poses. At its core, Vantage formulates viewpoint identification as a Bayesian optimization (BO) problem with a Gaussian process surrogate, optimizing a batched Upper Confidence Bound (q-UCB) acquisition strategy that suggests multiple informative  viewpoints to collect data. Since collecting demonstrations from different views is time consuming, this approach provides a sample-efficient solution. By iteratively fine-tuning on these ``vantage points,'' the policy learns to generalize across unseen perspectives, balancing view diversity with learning stability. 

We evaluate Vantage across a real-world Unitree D1 arm and a simulated Panda arm on RoboSuite manipulation tasks using multiple vision-guided policy architectures including Behavior Cloning (BC), Batch-Constrained Deep Q-learning (BCQ), BC Transformers (BCT), Action Chunking Transformers (ACT), and Diffusion Policies. Our experiments reveal substantial and consistent gains: for instance, a Diffusion policy on Pick\&Place improves from 37.01\% to 83.20\% success after fine-tuning with Vantage (+46.19\%). These gains emerge with only a handful of fine-tuning steps, making the approach not merely effective but also sample-efficient. Our key contributions can be summarized as:
\begin{enumerate}
    \item Proposing a framework for learning viewpoint-agnostic manipulation policies.
    \item Providing practically useful theoretical guarantees on the number of trials needed, stopping criteria for fine-tuning, and robustness under camera placement errors.
    \item Extensive empirical validation to demonstrate that Vantage consistently outperforms baselines.
\end{enumerate}


\section{RELATED WORK}

Early work sought robustness by training the policy with diverse variations in data such as changes in lighting, texture, background, and camera pose, so that the variations during training look similar to that of deployment (e.g., a shifted viewpoint)~\cite{tobin2017domain, sadeghi2016cad2rl}. Rather than physically capturing every camera pose, which is expensive, recent work attempts to synthesize alternative views through generative models. For instance, single-image novel-view diffusion models like Zero-1-to-3 \cite{liu2023zero1to3zeroshotimage3d}, ZeroNVS \cite{sargent2024zeronvszeroshot360degreeview} can render consistent target viewpoints. Robotics-specific pipelines such as VISTA \cite{tian2025viewinvariantpolicylearningzeroshot} then use these renders to train policies that are more stable under camera shifts, or to inject corrective visual perturbations via scene reconstructions. Cross-embodiment viewpoint augmentation \cite{chen2024roviaugrobotviewpointaugmentation} has likewise been explored to transfer skills with fewer new samples. These methods have focused on extending the viewpoint coverage only by a small angle. Scaling these to different or larger observation spaces requires training a separate view-generating diffusion model for each small region, which limits their practicality due to extremely high computational cost. Such limitations warrant selecting only important viewpoints.

Another reason to be deliberate about viewpoint selection comes from evidence in large in-the-wild datasets such as DROID \cite{khazatsky2024droid}, which was originally designed to demonstrate the value of broad scene diversity, including varied setups that implicitly alter viewpoints. Nevertheless, recent analyses caution that overly aggressive augmentation or randomization can degrade accuracy and do not consistently improve generalization \cite{zhang2017understanding, muandet2013domain, Gulrajani2021Search}, motivating more targeted data sampling strategies. Parallel work in domain generalization further highlights that excessive diversity between training and test distributions can harm performance~\cite{muandet2013domain, Gulrajani2021Search}. Together, these findings underscore the importance of actively selecting informative viewpoints.

In 3D reconstruction, there are attempts to actively gather viewpoints during training~\cite{Jayaraman2018LearningViewpointInvariance,Wu2023NeuralNBV,Lin2023NeuralImplicitActiveVision,Dhami2023MAPNBV,Dhami2023PredNBV,hou2024learning,liu2024splatraj,Chen2024GenNBV}. Modern solutions have utilized retrieval augmented priors~\cite{wright2024robust} to speed up the process. These 3D reconstruction methods focus on computing the similarity between two images as a proxy for information gain. However, this strategy does not transfer to manipulation settings, where distinct viewpoints do not necessarily translate into improved task performance. Rather than optimizing a single next view, our method based on BO emphasizes maximizing the downstream task performance of learned policies across unseen and dynamic camera settings. While BO has previously been used in robotics for informative path planning in mobile robots and controls~\cite{marchant2014bayesian,calandra2017bayesian} using Gaussian processes~\cite{williams2006gaussian} and Bayesian Hilbert maps~\cite{senanayake2017bayesian}, to the best of our knowledge, it has not been used for viewpoint selection in training robot manipulation policies.

\begin{figure*}[t]
  \centering
  \includegraphics[width=1\textwidth]{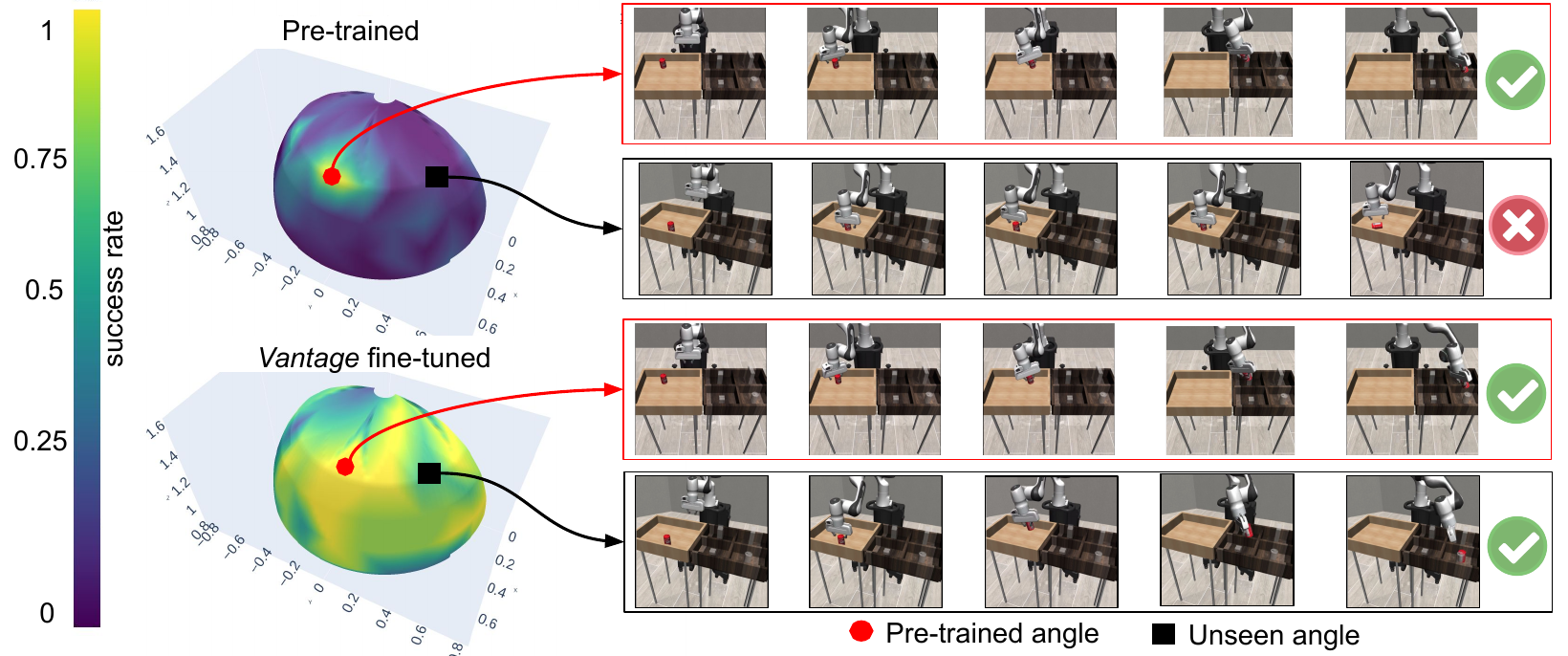}
    \caption{Success rate across observation space (i.e., the cradle of a hemisphere where the camera can be moved at deployment). The top row shows performance under the default pretraining viewpoint versus a novel test viewpoint. The bottom row compares the same hemisphere after \emph{Vantage} fine-tuning. While the pre-trained policy works around only where it was trained, \emph{Vantage}-fine-tuned model works almost everywhere.}
  \label{fig:compa}
\end{figure*}

\section{METHODOLOGY}
\label{Methodolgy}

\subsection{Viewpoint Selection as Optimization}

At the heart of our framework lies a simple question: \emph{from which camera angles should the robot learn in order to be robust across all others?} We define a viewpoint, $\theta$, in the 3D space as the camera placement around the robot (see Fig.~\ref{fig:intro}). Given a pre-trained policy $\pi$, our goal is to identify a \textbf{vantage point}, $\theta^{\text{vantage}}$, from the set of all training points $\theta_{\text{train}}$, such that fine-tuning $\pi$ with data from this view maximizes the average task success rate
$J\bigl(\pi_{\theta_{\text{train}}}; \Theta_{\text{test}}\bigr),$ across rollouts from a set of test viewpoints, $\Theta_{\text{test}}$. Thus, we want to learn,
\begin{equation}
  \theta^\text{vantage}
  = \argmax_{\theta_\text{train} \in \Theta_\text{train}}
      J\!\left(\pi_{\theta_\text{train}}(\Theta_\text{test})\right),
\end{equation}
where, $\pi_{\theta_\text{train}}$ denotes the policy obtained by fine-tuning the pre-trained model on the original dataset together with demonstrations collected from viewpoint $\theta_\text{train}$.

Since the mapping from viewpoint to downstream task success has no analytic gradients or closed-form structure, gradient-based optimization and exhaustive enumeration are impractical. 
This setting aligns precisely with the regime addressed by black-box optimization methods, motivating our choice of Bayesian optimization for efficient viewpoint selection in the continuous observation space of $\Theta_\text{test}$. We will later establish in Theorem~\ref{thm:1} that this formulation ensures the cumulative regret of our search remains sublinear, and hence, the number of poor viewpoint choices grows much slower than the number of fine-tuning attempts.


\subsection{Bayesian Optimization for Viewpoint Search}

We model the mapping from training viewpoint $\theta_\text{train}$ to performance $J$ using a Gaussian Process (GP):
\begin{equation}
  J(\pi_{\theta_\text{train}})
  \sim \mathcal{GP}\!\big(\mu(\theta_\text{train}),
  k(\theta_\text{train}, \theta_\text{train}^\prime)\big),
  \label{eq:gp}
\end{equation}
with squared-exponential kernel $k(\cdot,\cdot)$. This surrogate captures our uncertainty: the posterior mean $\mu$ predicts expected manipulation task performance, while the covariance induced by $k$ encodes uncertainty in predicting task performance due to similarity between camera positions. This uncertainty allows us to decide \emph{where to look next} without exhaustively training on every view. See~\cite{senanayake2024role} for a review on uncertainty in robotics.  

Classical BO would select one new viewpoint at a time, but policy fine-tuning can be parallelized, thus we use a batched acquisition strategy, q-Upper Confidence Bound (q-UCB)~\cite{wilson1712reparameterization}. For a batch $\Theta_q=\{\theta_1,\dots,\theta_q\}$, q-UCB scores each candidate set of viewpoints by combining exploitation (high predicted mean) with exploration (high uncertainty):
\begin{equation}
\begin{split}
\Theta_q^\text{next}
&= \argmax_{\Theta_q \subset \Theta_\text{train}} 
    \alpha_{\mathrm{qUCB}}(\Theta_q) \\
&= \argmax_{\Theta_q \subset \Theta_\text{train}} 
    \mathbb{E}_{\tilde \Theta \sim 
      \mathcal{N}\!\bigl(\mu(\Theta_q),\,\tfrac{\beta\pi}{2}\,\Sigma(\Theta_q)\bigr)} \\
&\quad \Bigl[\max_{i=1,\dots,q}\bigl(\mu(\theta_i) 
      + \lvert \tilde \Theta_i - \mu(\theta_i)\rvert\bigr)\Bigr].
\end{split}
\label{eq:qucb}
\end{equation}

Here, $ \mathcal{N}$ is a normal distribution and $\beta>0$ controls exploration vs exploitation: large $\beta$ encourage sampling from uncertain regions, while smaller ones favor viewpoints already predicted to perform well. Please refer Appendix A of~\cite{wilson1712reparameterization} for the detailed derivation of~(\ref{eq:qucb}). As we will show in Theorem~\ref{thm:2}, this acquisition strategy ensures that the average task success of policies fine-tuned with \emph{Vantage} converges toward the optimal viewpoint, with error diminishing at rate $O(T^{-1/2})$ with $T$ steps.

\subsection{Iterative Fine-Tuning Procedure}
The full Vantage procedure unfolds as follows:
\begin{enumerate}
    \item \textbf{Initialization.} Sample $q$ random viewpoints, fine-tune, and record success rates. Fit an initial GP surrogate. 
    \item \textbf{Vantage training loop.} In iteration $T$, update the GP with accumulated data, select a new batch of viewpoints via q-UCB, fine-tune at each, evaluate, and expand the dataset. 
    \item \textbf{Final selection.} After $N$ rounds, choose the fine-tuned policy with the highest observed success rate.
\end{enumerate}

This iterative process is given in detail in Algorithm~\ref{algo:vantage}. Conceptually, each round of fine-tuning both improves the policy and refines our estimate of the underlying function $f(\theta)$, steering the search toward increasingly informative vantage points. Moreover, practical deployment often faces camera placement noise due to calibration errors or mechanical offsets; Theorem~\ref{thm:robustness} will show that the BO trajectory is provably robust to such Gaussian perturbations, ensuring stability of our method in real-world settings.

\subsection{Theoretical Guarantees}
Beyond empirical performance, Vantage provides provable guarantees that offer practical guidance for engineers. These results not only establish formal learning-theoretic bounds, but also translate into actionable principles for system design: how many trials are needed, when to stop fine-tuning, and how precise the hardware setup must be. By framing the guarantees in terms of data efficiency, convergence behavior, and robustness to camera placement errors, Vantage bridges the gap between abstract theory and the concrete engineering decisions required for real-world robot deployment.  

\begin{theorem}[Efficiency of viewpoint selection]\label{thm:1}
Let $f:\Theta \to [0,1]$ denote the mapping from training viewpoints to average success rates, drawn from a GP prior with kernel $k$. Assume observations $y_t = f(\theta_t)+\varepsilon_t$ with $\varepsilon_t$ sub-Gaussian. Running $q$-UCB for $T$ rounds yields cumulative regret
\[
R(T) = \sum_{t=1}^T \sum_{j=1}^q \bigl[f(\theta^*) - f(\theta_{t,j})\bigr]
= O\!\left(\sqrt{qT\,\gamma_{qT}\,\beta_T}\right),
\]
with probability at least $1-\delta$, where $\delta$ is a user chosen failure probability, and $\theta^*=\arg\max_{\theta\in\Theta} f(\theta)$, $\gamma_{qT}$ is the batch information gain, and $\beta_T$ is confidence.
\end{theorem}

\begin{proof}[Proof sketch]
The argument follows Theorem~2 of~\cite{Srinivas_2012}. Each UCB choice ensures that the instantaneous regret is bounded by a multiple of the posterior standard deviation. Summing over $qT$ queries and applying Cauchy–Schwarz with the information gain bound yields the sublinear regret rate. Full proof Appendix~\ref{thm:gpucb}.
\end{proof}

This GP-UCB regret bound assures engineers that wasted trials grow slowly compared to total trials, meaning only a handful of optimized camera placements (e.g., 10–15) are typically sufficient rather than hundreds. This lets practitioners budget resources such as robot hours and data collection costs up front.

\begin{theorem}[Success rate convergence]\label{thm:2}

Under the same assumptions, with probability at least $1-\delta$,
\[
\begin{aligned}
\frac{1}{T}\sum_{t=1}^T J(\pi_{\theta_t})
&\;\ge\;
J(\pi_{\theta^*})
- O\!\Bigl(\sqrt{\tfrac{\gamma_T\,\beta_T}{T}}\Bigr)
\end{aligned}
\]

\end{theorem}

\begin{proof}[Proof sketch]
From Theorem~\ref{thm:1}, cumulative regret is $O(\sqrt{T\,\gamma_T\,\beta_T})$. Dividing by $T$ gives the average success rate bound, which shows that the mean performance of policies fine-tuned with Vantage converges to the optimum. Full proof is provided in the Appendix~\ref{thm:avgconv}. 
\end{proof}

These results guarantee that, as the number of fine-tuning rounds increases, Vantage converges to near-optimal camera viewpoints while requiring only sublinear exploration of the vast space. Engineers can use this to define stopping rules to avoid over-training and save time. Once success rates plateau, further fine-tuning offers diminishing returns.

\begin{table*}[t]
  \centering
  \setlength{\tabcolsep}{4pt}
  \renewcommand{\arraystretch}{0.95}
  \small
    \caption{Success rates (\%) of policies across tasks under three camera settings: 
    \emph{Default} (trained viewpoint), \emph{$\Theta$} (set of all allowed viewpoints: quarter-sphere in front of the robot), and \emph{Dynamic} (continuous motion). 
    Results compare baseline training, grid/random augmentation, and our proposed method \emph{Vantage}.}
  \label{tab:table_1}
  \resizebox{\linewidth}{!}{%
  \begin{tabular}{llcccccccccccc}
    \toprule
      & \multirow{2}{*}{Camera}
      & \multicolumn{4}{c}{BC} & \multicolumn{4}{c}{Diffusion} & \multicolumn{4}{c}{BCT} \\
      \cmidrule(lr){3-6} \cmidrule(lr){7-10} \cmidrule(lr){11-14}
      & Placement & Base & Grid & Standard & Vantage
        & Base & Grid & Standard & Vantage
        & Base & Grid & Standard & Vantage\\
    \midrule
    \multirow{3}{*}{Lift}
      & Default   & 100.0 & 100.0 & 100.0 & 100.0
                   & 100.0 & 100.0 & 100.0 & 100.0 
                   & 100.0 & 100.0 & 100.0 & 100.0 \\
      & $\Theta$  & 6.90  & 9.18  & 23.15 & \textbf{23.50}
                   & 50.40 & 58.57 & 65.82 & \textbf{66.10} 
                   & 6.71 & 10.12 & 10.12 & \textbf{10.44} \\
      & Dynamic   & 91.66 & 93.33 & \textbf{100.0} & \textbf{100.0}
                   & 98.33 & \textbf{100.0} & \textbf{100.0} & \textbf{100.0} 
                   & 91.66 & 97.33 & 97.33 & \textbf{98.12} \\
    \midrule
    \multirow{3}{*}{Pick Place}
      & Default   & 70.0 & 70.0 & 70.0 & 70.0
                   & 90.0 & 90.0 & 90.0 & 90.0 
                   & 90.0 & 90.0 & 90.0 & 90.0 \\
      & $\Theta$  & 0.80 & 1.04 & 2.53 & \textbf{3.90}
                   & 37.01 & 22.09 & 66.11 & \textbf{83.20} 
                   & 1.69 & 1.70 & 1.94 & \textbf{1.94} \\
      & Dynamic   & 3.00 & 4.23 & 7.66 & \textbf{9.88}
                   & 88.33 & 71.54 & 92.33 & \textbf{97.16} 
                   & 13.33 & 16.17 & 18.04 & \textbf{18.04} \\
    \midrule
    \multirow{3}{*}{Square}
      & Default   & 30.0 & 20.0 & \textbf{30.0} & 20.0
                   & 60.0 & 60.0 & 60.0 & \textbf{70.0} 
                   & 50.0 & 50.0 & 60.0 & \textbf{60.0} \\
      & $\Theta$  & 0.24 & 0.61 & 0.74 & \textbf{1.00}
                   & 2.38 & 2.38 & 12.08 & \textbf{14.80} 
                   & 0.74 & 1.49 & 1.49 & \textbf{1.49} \\
      & Dynamic   & 0.00  & 0.00  & 0.00  & 0.00
                   & 8.33 & 8.33 & 40.0 & \textbf{54.66} 
                   & 0.00 & 0.51 & 0.51 & \textbf{0.51} \\
    \bottomrule
  \end{tabular}}
\end{table*}

\begin{algorithm}[t]
\caption{\textit{Vantage}}
\label{algo:vantage}
\begin{algorithmic}

\State \textbf{Step 1: Gather Initial Data}
\State Sample $q$ random viewpoints $\{\theta_\text{train}^{(j)}\}_{j=1}^q$, where each $\theta_\text{train}^{(j)} \in \Theta_\text{train}$
\State Generate manipulation datasets $\{\mathcal{D}^{(j)}\}_{j=1}^q$, from robot trials or simulation at viewpoints $\{\theta_\text{train}^{(j)}\}_{j=1}^q$
\State Fine-tune the original policy on $\mathcal{D}_j$ independently
\State Evaluate the fine-tuned models across $\Theta$ to obtain success rates $\{J_i\}_{i=1}^q$
\State Initialize a historical dataset $\mathcal{D}_{gp} \gets  \{(\theta_j, J_i)\}_{i=1}^q$

\State \textbf{Step 2: Vantage training loop}
\For{$i = 1$ to $N$}
    \State Use $D_{gp}$ and Bayesian optimization to select q new angles $\{\theta_{\text{new},j}\}_{j=1}^q = (\theta_h^{\text{new},j}, \theta_v^{\text{new},j}) \in \Theta$
    \State Generate datasets $\mathcal{D}_{\text{new},j}$ at $\theta_{\text{new},j}$ 
    \State Fine-tune the original model separately on $\mathcal{D}_{\text{new},j}$
    \State Evaluate the model fine-tuned at $\theta_{\text{new}}$ to obtain $J_{\text{new}}$
    \State Update $\mathcal{D}_{gp} \gets \mathcal{D}_{gp} \cup \{(\theta_{\text{new},j}, J_{\text{new},j})\}$
\EndFor

\State \textbf{Step 3: Final Selection}
\State $\theta^{vantage} \gets \displaystyle \arg\max_{\theta \in \Theta}\, J(\pi_\theta)$ 

\end{algorithmic}
\end{algorithm}

\begin{theorem}[Robustness under camera placement error]\label{thm:robustness}
Let $f:\mathbb{R}^d \to \mathbb{R}$ be modeled with a Gaussian process prior using the squared--exponential covariance. 
Suppose Bayesian optimization selects query viewpoints $x_t \in \mathbb{R}^d$, but the executed viewpoints are 
$\tilde{x}_t = x_t + \varepsilon_t$, where $\varepsilon_t \sim \mathcal{N}(0,\sigma_x^2 I_d)$ are i.i.d.\ perturbations due to camera placement error of standard deviation $\sigma$. 
Define the effective objective $g(x) = \mathbb{E}[\,f(x+\varepsilon)\,].$ Then $g$ is again governed by a squared exponential GP with re-parametrized hyperparameters. 
Hence, \emph{Bayesian optimization with squared--exponential priors is robust to such input noise}: 
the optimization trajectory is unchanged relative to the noise-free case.
\end{theorem}

\begin{proof}[Proof sketch]
Since $g(x)=\mathbb{E}_\varepsilon[f(x+\varepsilon)]$ is a linear functional of a GP, $g$ is again a GP.  
Its covariance is
\[
k_g(x,x')=\mathbb{E}_{\varepsilon,\varepsilon'}\big[k(x+\varepsilon,\,x'+\varepsilon')\big],
\]
with $k$ the squared--exponential kernel.  
Because $\varepsilon-\varepsilon'\sim\mathcal{N}(0,2\sigma_x^2I_d)$, this expectation reduces to a Gaussian integral
\[
k_g(x,x')=\sigma_f^2\,\mathbb{E}_\delta\!\Big[\exp\!\big(-\tfrac{\|(x-x')+\delta\|^2}{2\ell^2}\big)\Big],\quad \delta\sim\mathcal{N}(0,2\sigma_x^2I_d).
\]
Applying the standard Gaussian moment identity yields
\[
k_g(x,x')=\sigma_f^2\Big(\tfrac{\ell^2}{\ell^2+2\sigma_x^2}\Big)^{\!d/2}\!
\exp\!\Big(-\tfrac{\|x-x'\|^2}{2(\ell^2+2\sigma_x^2)}\Big).
\]
Thus $g$ has the same squared--exponential form with enlarged lengthscale and reduced marginal variance as stated. Full proof is provided in the Appendix~\ref{thm:robustness_appendix}.  
\end{proof}

This theorem highlights that engineers do not need ultra-precise camera placement; even with modest calibration errors or approximate positioning, performance remains stable, highlighting the method's practical simplicity.

\section{EXPERIMENTAL RESULTS}
\label{sec:exp_results}

We design experiments to rigorously evaluate \emph{Vantage} across both simulation and real-robot settings. Our goal is to assess whether strategic viewpoint selection improves generalization to novel viewpoints and dynamic camera motion. In particular, we seek to answer the following four questions:

\begin{enumerate}
    \item How does a policy fine-tuned with \emph{Vantage} compare to standard and grid  based fine-tuning? (Sec~\ref{exp:1}) 
    \item Does \emph{Vantage} enable more efficient discovery of informative camera viewpoints? (Sec~\ref{exp:2})
    \item How well do \emph{Vantage} trained policies work in dynamic viewpoint changes? (Sec~\ref{exp:3})
    \item Do the improvements achieved with \emph{Vantage} extend to real robot deployments? (Sec~\ref{exp:4})
\end{enumerate}

\begin{figure}[t]
    \centering
    \includegraphics[width=0.24\linewidth]{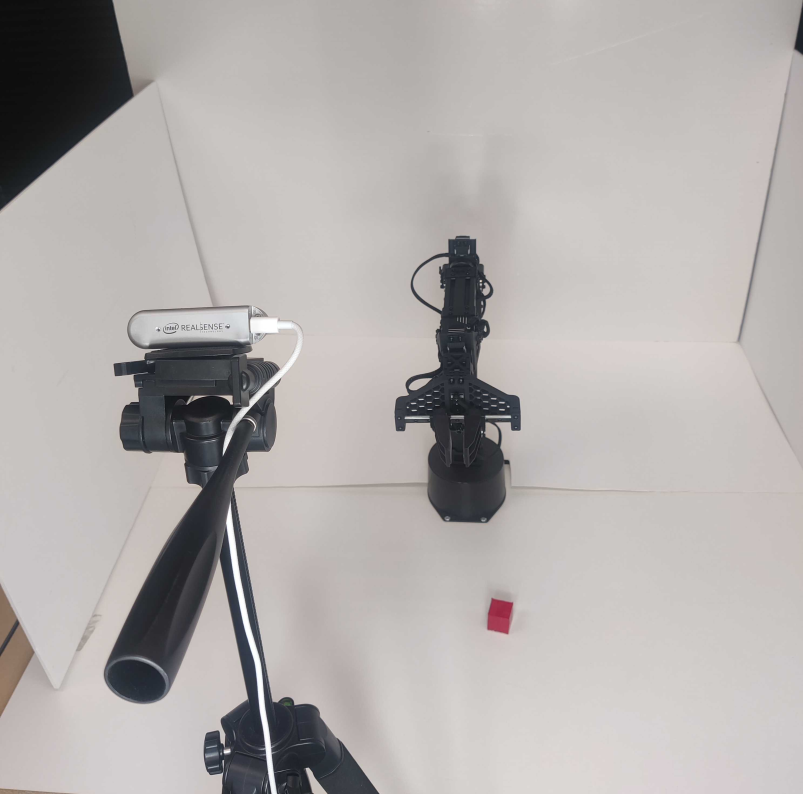}
    \includegraphics[width=0.24\linewidth]{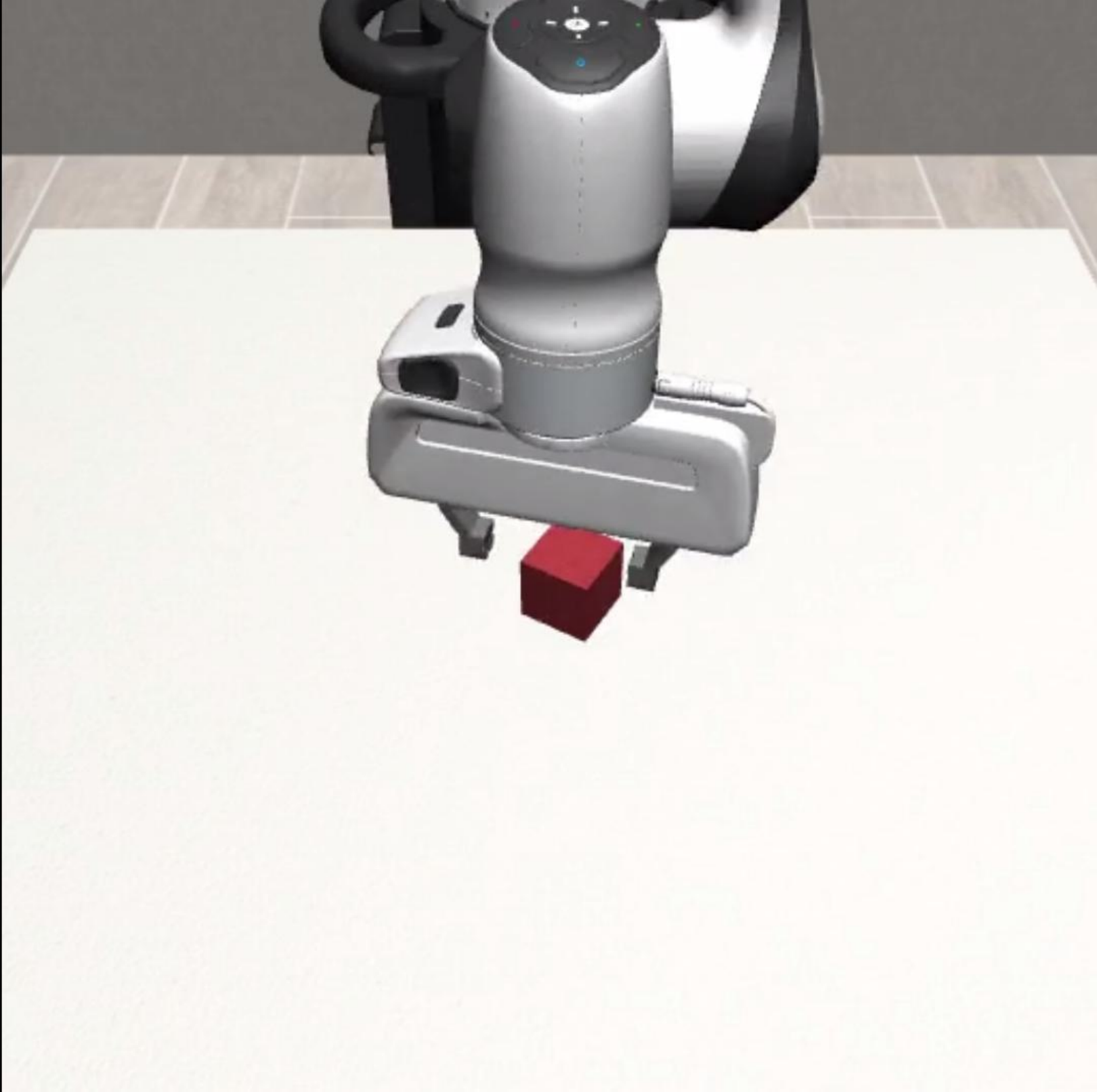}
    \includegraphics[width=0.24\linewidth]{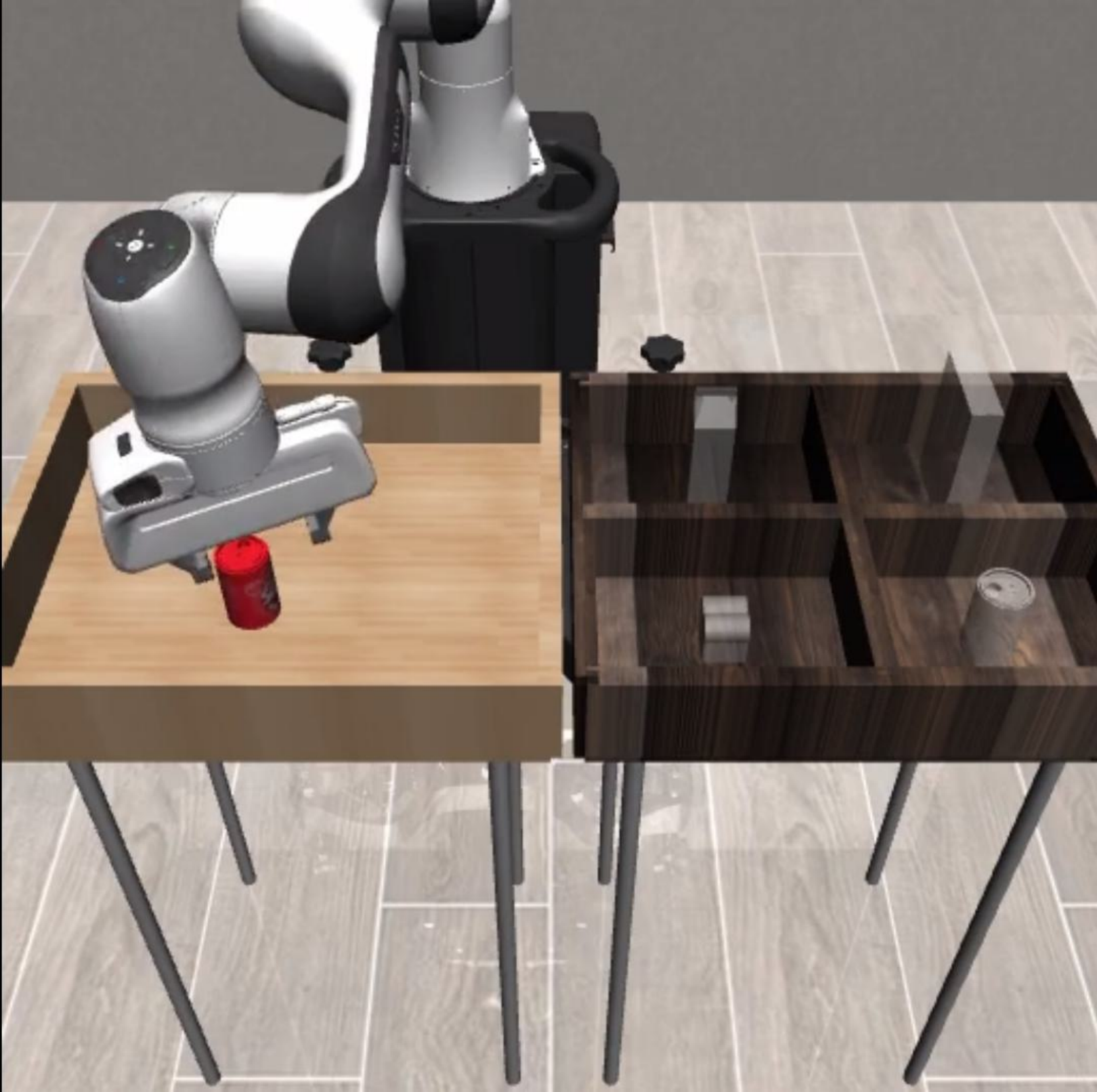}
    \includegraphics[width=0.24\linewidth]{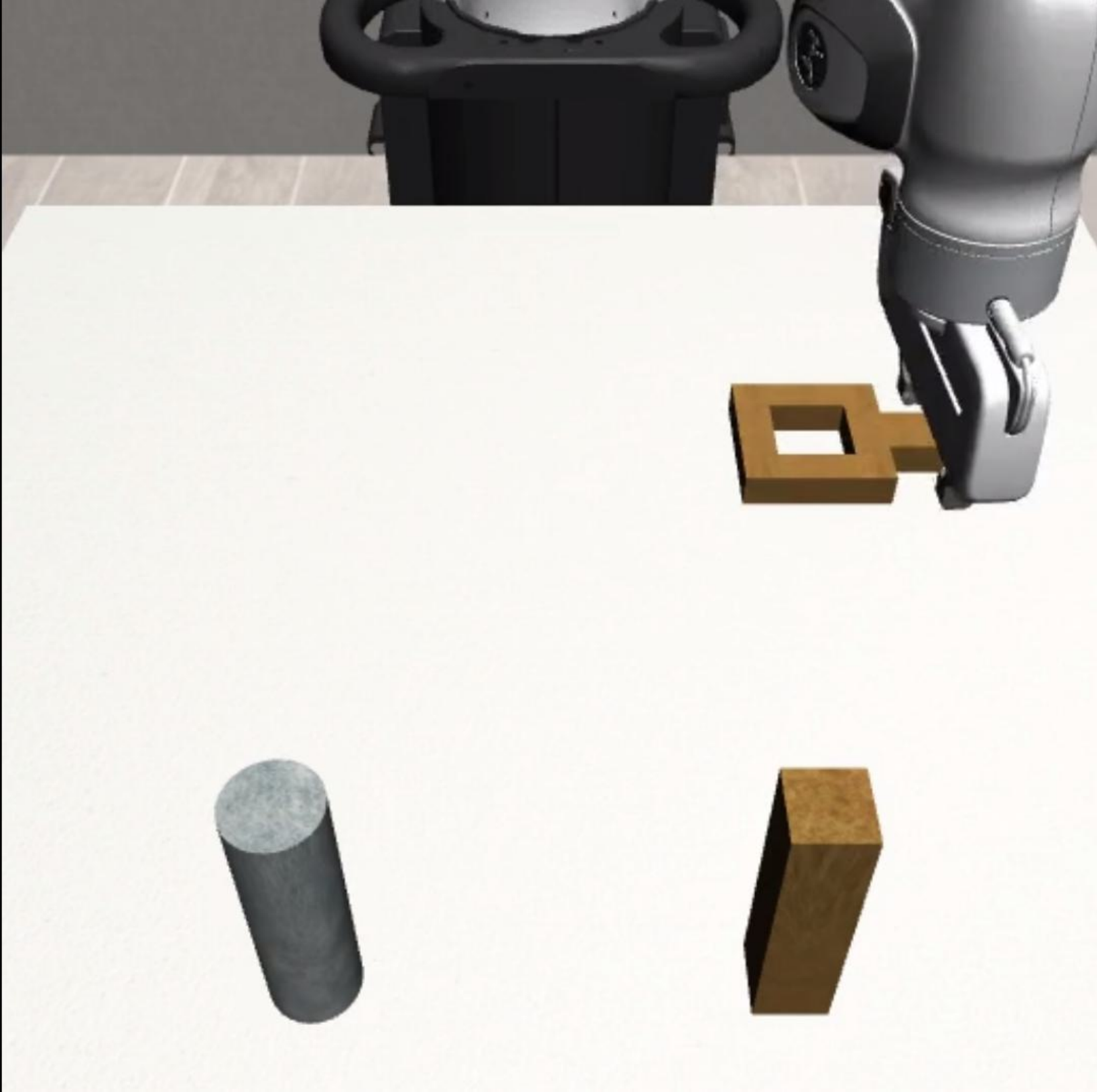}
    \caption{Experiment setups for \emph{Vantage}: (First) external camera placement with Unitree D1 arm, (second, third, fourth) RoboSuite environments for benchmark tasks.}
    \label{fig:real_robot}
\end{figure}

\subsection{Setup}
\textbf{Simulation:} We conduct experiments in the RoboSuite framework~\cite{zhu2025robosuitemodularsimulationframework} using RoboMimic benchmark datasets~\cite{robomimic2021}, which provide demonstrations for three standard manipulation tasks: \emph{Lift}, \emph{Square}, and \emph{Pick \& Place}. To test generality across model classes, we benchmark \emph{Vantage} with policies spanning diverse architectures and inductive biases, including BC~\cite{bain1995framework}, BCQ~\cite{fujimoto2019off}, BCT~\cite{robomimic2021}, and Diffusion~\cite{chi2024diffusionpolicyvisuomotorpolicy}. Using demonstrations collected from the default viewpoint in RoboSuite, we train each of the aforementioned policies. These serve as the base versions of the policies. We define the quarter-sphere in front of the robot (see Fig~\ref{fig:intro}) as the space of allowed viewpoints ($\Theta$). Then the policy is fine-tuned from different angles selected by various viewpoint selection strategies (default fine-tuning, grid search and \emph{Vantage}).

\begin{figure}[t]
    \centering
  \includegraphics[width=0.8\linewidth]{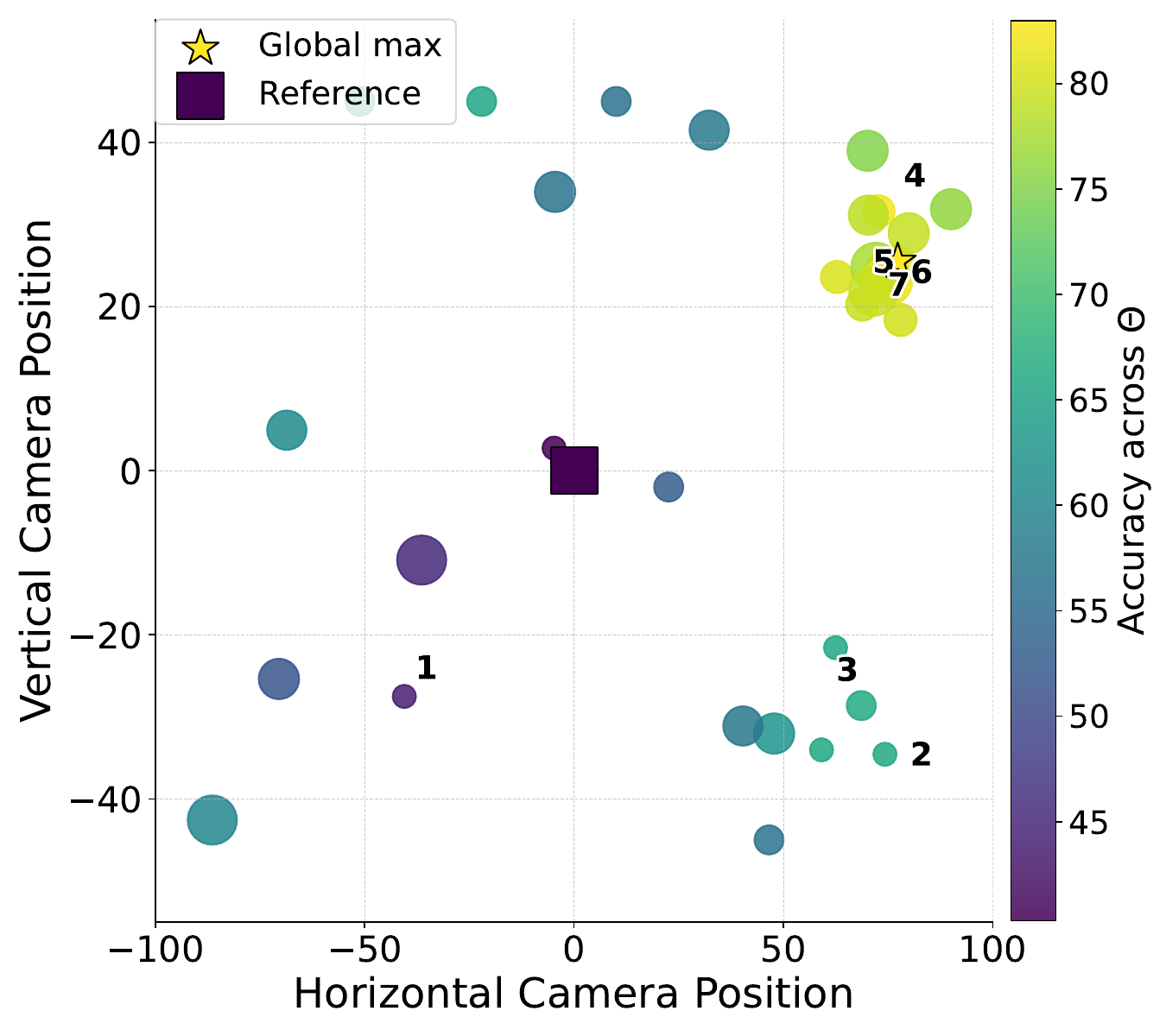}
    \caption{Vantage iterations for viewpoint selection on \emph{Pick \& Place}. Each point corresponds to a candidate camera position evaluated during fine-tuning, with color and size indicating success rate and iteration, respectively. The purple square marks the default training viewpoint, while the yellow star denotes the globally best-performing \emph{vantage point}. The search progressively concentrates around informative regions, converging to the optimum by iteration 7.}
    \label{fig:camera_pos}
\end{figure}

\begin{figure*}[t]
  \centering
  \begin{subfigure}[t]{0.8\textwidth}
    \centering
    \includegraphics[width=\linewidth]{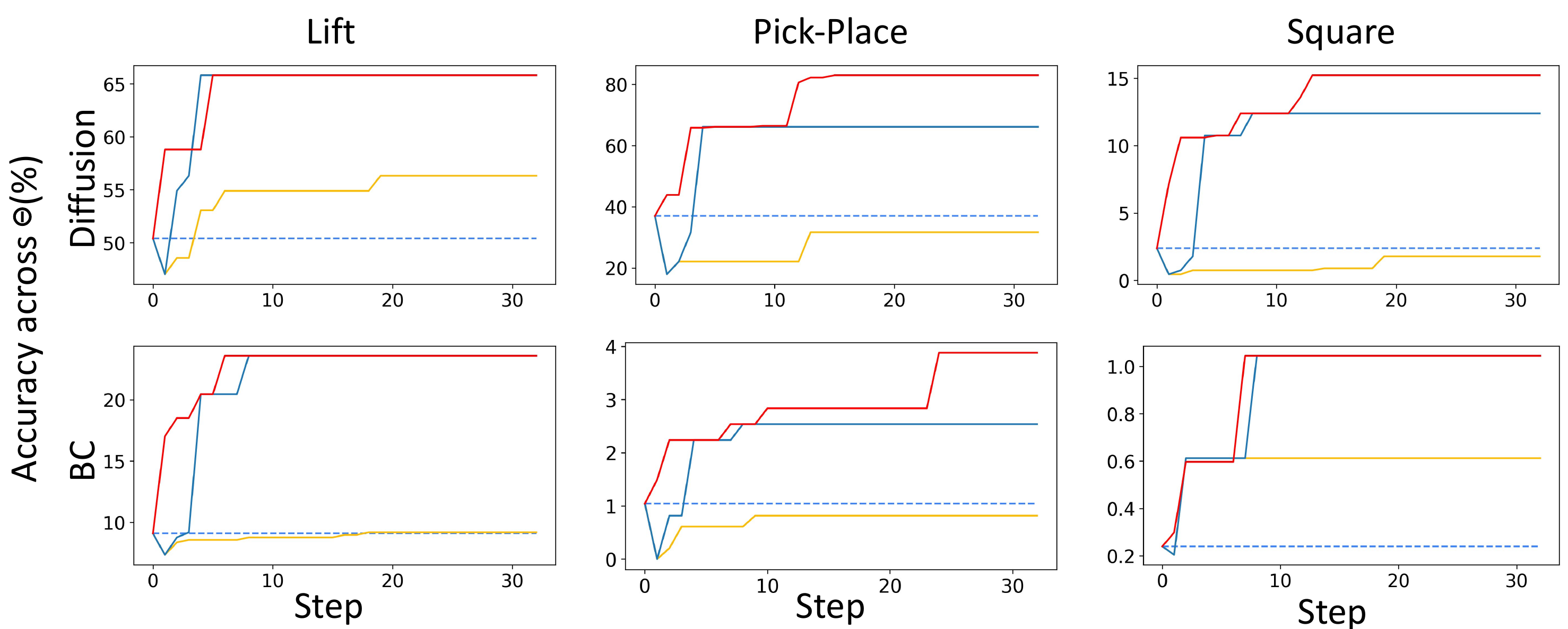}
  \end{subfigure}

  \begin{subfigure}[t]{0.55\textwidth}
    \centering
    \includegraphics[width=\linewidth]{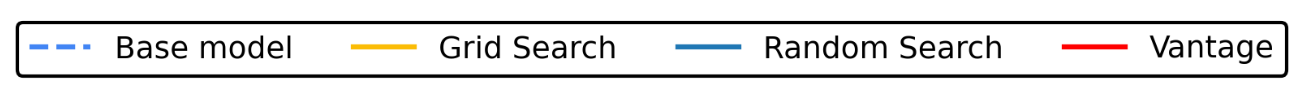}
  \end{subfigure}

    \caption{Convergence of best accuracy across viewpoint space $\Theta$ for Diffusion policies on \emph{Lift}, \emph{Pick \& Place}, and \emph{Square}. Each curve tracks the best-performing policy over optimization steps under different search strategies: base model (dashed), grid search, random search, and \emph{Vantage}. \emph{Vantage} (red) consistently converges faster and achieves higher final accuracy with fewer iterations, demonstrating its sample efficiency compared to exhaustive or heuristic approaches.}
  \label{fig:hyperplane_comparison_grid}
\end{figure*}

\textbf{Real robot:} To study sim-to-real transfer, we deploy \emph{Vantage} on a Unitree D1 7-DoF arm (see Fig~\ref{fig:real_robot}). We fine-tune an ACT based~\cite{zhao2023learning} visuomotor policy on a reaching task while systematically varying camera placements around the workspace. This setup allows us to test whether the robustness observed in simulation also persists under real-world sensing noise and actuation dynamics.

We evaluate under two regimes that reflect realistic deployment scenarios:

\begin{enumerate}
    \item \textbf{Diverse static viewpoints:} Policies are tested across a quarter-sphere of feasible camera placements, probing generalization to novel but fixed viewpoints not seen during fine-tuning.
    \item \textbf{Dynamic viewpoints:} Policies are tested under continuous camera motion during inference, simulating deployment with mobile manipulators or head-mounted cameras where viewpoints shift over time.
\end{enumerate}

\subsection{Comparison with Viewpoint Search Baselines}
\label{exp:1}
To contextualize performance, we compare against three standard baselines. 

\begin{enumerate}
    \item \textbf{Default angle accuracy} evaluates policies exclusively on the original pretraining viewpoint. 
    \item \textbf{Grid search fine-tuning} systematically samples viewpoints over a discretized grid, offering broader coverage but at significant cost. 
    \item \textbf{Standard viewpoint fine-tuning} samples views randomly, providing diversity but often wasting training budget on uninformative perspectives.  
\end{enumerate}

As summarized in Table~\ref{tab:table_1}, \emph{Vantage} consistently matches or outperforms these baselines across all tasks and policy families. The most dramatic gains appear in \emph{Pick \& Place} with a diffusion policy, where success improves from $37.01\%$ to $83.20\%$ (+$46.19\%$). Similar trends are observed for BC and BCT, demonstrating robustness across architectures.

\begin{table}[t]
  \centering
  \setlength{\tabcolsep}{4pt}
  \renewcommand{\arraystretch}{0.95}
  \small
     \caption{Comparison of success rates (\%) across augmentation strategies. 
    Viewpoint Randomization with 1 or 8 additional views provides limited or inconsistent gains, sometimes degrades performance. 
    Generated Augmentations with 5 synthetic views offers moderate improvements but remains less effective than optimized selection. 
    In contrast, \emph{Vantage} achieves competitive or superior performance with only a single optimized viewpoint.}
  \label{tab:augmentation_strategies}
  \resizebox{1\linewidth}{!}{%
  \begin{tabular}{lcccc}
    \toprule
    Task \& Policy & Viewpoint & Viewpoint & Generated & Vantage \\
    & Randomization & Randomization & Augmentation &  \\
      & (1 view) & (8 views) & (5 views) & (1 view) \\
    \midrule
    Pick Place (BC)        & {\normalsize 2.1}  & {\normalsize 4.3}  & {\normalsize 1.2}  & { \normalsize \textbf{3.9}}  \\
    Pick Place (Diffusion) & {\normalsize 63.8} & {\normalsize 42.6} & {\normalsize 52.3} & {\normalsize \textbf{83.2}} \\
    Square (BC)            & {\normalsize 0.9}  & {\normalsize 0.3}  & {\normalsize 0.3}  & {\normalsize \textbf{1.0}}  \\
    Square (Diffusion)     & {\normalsize 12.3} & {\normalsize \textbf{17.5}} & {\normalsize 14.6}  & {\normalsize 14.8} \\
    Real. (ACT)     & {\normalsize 36.0} & {\normalsize 30.0} & ---  & {\normalsize \textbf{44.0}} \\
    \bottomrule
  \end{tabular}}
\end{table}

\subsection{Efficiency of Vantage for Viewpoint Discovery}
\label{exp:2}
Unlike grid and random strategies, which waste budget on redundant or uninformative views, \emph{Vantage} leverages Bayesian optimization to quickly identify promising viewpoints. As visualized in Fig.~\ref{fig:camera_pos}, the search concentrates in informative regions after only a few iterations. In practice, convergence occurs within $10$–$15$ trials, a fraction of the budget required by exhaustive search. This efficiency is further reflected by the rapid performance gains as shown in Fig.~\ref{fig:hyperplane_comparison_grid}. The q-UCB calculation to determine the next viewpoint requires less than 30 seconds, while the combined process of data collection and fine-tuning completes within $\approx$1 hour, underscoring the overall time efficiency of the approach.

Quantitatively, Table~\ref{tab:augmentation_strategies} further highlights that \emph{Vantage} outperforms heavy augmentation. For example, in \emph{Pick Place} with a diffusion policy, viewpoint randomization with 8 additional views degrades performance, while \emph{Vantage} reaches $83.2\%$ using only a handful of views. This efficiency is also evident in the optimization landscape of Fig.~\ref{fig:compa}, which shows how \emph{Vantage} reshapes the performance contours.


\subsection{Performance Under Dynamic Viewpoints}
\label{exp:3}
A central test of real-world utility is robustness under dynamic camera motion. \emph{Vantage} fine-tuned policies sustain performance even when camera viewpoints change continuously at inference time. Dynamic viewpoint evaluation simulates settings such as mobile manipulators or head-mounted cameras, where the robot cannot rely on a fixed perspective. As shown in Table~\ref{tab:table_1}, policies fine-tuned with \emph{Vantage} consistently exhibit higher robustness under dynamic evaluation compared to baselines. For example, on \emph{Pick \& Place} with a Diffusion policy, success under dynamic motion increases from $88.3\%$ to $97.2\%$ with \emph{Vantage}. On the more challenging \emph{Square} task, Diffusion policy with \emph{Vantage} improves in accuracy from $8.3\%$ to $54.7\%$, demonstrating resilience to viewpoint ambiguity. Even simple BC policies, which otherwise collapse under motion (e.g., $3.0\%$ $\to$ $9.9\%$ on \emph{Pick \& Place}), benefit from targeted viewpoint selection. 

Table~\ref{tab:augmentation_strategies} further highlights this point by comparing \emph{Vantage} against viewpoint randomization (VR) \cite{tobin2017domain} and generated augmentation strategies akin to VISTA\cite{tian2025viewinvariantpolicylearningzeroshot}. While VR with many views can sometimes improve performance (e.g., Diffusion on \emph{Square}, $12.3\% \to 17.5\%$), it often requires large data budgets and may even degrade performance due to unstructured variability (VR with 8 views on \emph{Pick \& Place} reduces accuracy from $63.8\%$ to $42.6\%$). In contrast, \emph{Vantage} achieves $83.2\%$ success with only a handful of optimized views, demonstrating sample efficiency and superior robustness. By explicitly optimizing where the robot learns to see from, \emph{Vantage} induces viewpoint invariant representations that remain stable under continuous camera motion.

\subsection{Does Vantage work on real-world robots?}
\label{exp:4}

Finally, we evaluate whether the gains achieved in simulation translate to physical hardware. We deploy \emph{Vantage} on a Unitree D1 7-DoF arm and fine-tune an ACT based visuomotor policy for a reaching task under varying camera placements. As shown in Table~\ref{tab:augmentation_strategies}, naive augmentation strategies such as VR provide only limited or inconsistent improvements, with success rates ranging from $30.0\%$ to $36.0\%$. In contrast, \emph{Vantage} achieves a success rate of $44.0\%$ using only a single optimized viewpoint, outperforming the other methods. Because camera placement is imprecise in practice, in simulation we also inject realistic placement noise and measure the post-\emph{Vantage} accuracy of the best-performing model. This experiment revealed that, even under such placement errors, models trained with \emph{Vantage} remained within $\approx$5\% margin of error. These results confirm that strategic viewpoint selection not only improves robustness in simulation but also extends to real-world sensing and actuation conditions, where camera calibration errors, occlusions, and lighting variability pose additional challenges.

\section{CONCLUSIONS}

We introduced \emph{Vantage}, a framework that treats camera viewpoint selection as a continuous optimization problem and leverages Bayesian optimization to identify informative perspectives for fine-tuning visuomotor policies. By explicitly balancing exploration of novel viewpoints with exploitation of high-performing ones, \emph{Vantage} provides a principled and sample-efficient alternative to naive data augmentation or exhaustive viewpoint search. Our experiments across search benchmarks and multiple policy architectures demonstrate that \emph{Vantage} consistently improves robustness to viewpoint shifts, often by large margins. In particular, we observe substantial gains even with limited fine-tuning, as well as strong generalization under dynamic camera evaluation. These results highlight that carefully chosen training viewpoints can be as important as the underlying policy architecture in achieving viewpoint-agnostic manipulation.


\begin{thebibliography}{10}

\bibitem{lee2022uncertainty}
Soomin Lee, Le~Chen, Jiahao Wang, Alexander Liniger, Suryansh Kumar, and Fisher Yu.
\newblock Uncertainty guided policy for active robotic 3d reconstruction using neural radiance fields.
\newblock In {\em IEEE International Conference on Robotics and Automation (ICRA)}, 2022.

\bibitem{zhang2023affordance}
Xuechao Zhang, Dong Wang, Sun Han, Weichuang Li, Bin Zhao, Zhigang Wang, Xiaoming Duan, Chongrong Fang, Xuelong Li, and Jianping He.
\newblock Affordance-driven next-best-view planning for robotic grasping.
\newblock In {\em Conference on Robot Learning (CoRL)}, 2023.

\bibitem{zhang2017understanding}
Chiyuan Zhang, Samy Bengio, Moritz Hardt, Benjamin Recht, and Oriol Vinyals.
\newblock Understanding deep learning requires rethinking generalization.
\newblock In {\em International Conference on Learning Representations (ICLR)}, 2017.

\bibitem{muandet2013domain}
Kunal Muandet, David Balduzzi, and Bernhard Schölkopf.
\newblock Domain generalization via invariant feature representation.
\newblock {\em Journal of Machine Learning Research}, 17(57), 2013.

\bibitem{atharva25}
Atharva Gundawar, Som Sagar, and Ransalu Senanayake.
\newblock Pac bench: Do foundation models understand prerequisites for executing manipulation policies?
\newblock {\em arXiv preprint arXiv:2506.23725}, 2025.

\bibitem{kim2024openvla}
Moo~Jin Kim, Karl Pertsch, Siddharth Karamcheti, Ted Xiao, Ashwin Balakrishna, Suraj Nair, Rafael Rafailov, Ethan Foster, Grace Lam, Pannag Sanketi, et~al.
\newblock Openvla: An open-source vision-language-action model.
\newblock {\em arXiv preprint arXiv:2406.09246}, 2024.

\bibitem{tobin2017domain}
Josh Tobin, Rachel Fong, Alex Ray, Jonas Schneider, Wojciech Zaremba, and Pieter Abbeel.
\newblock Domain randomization for transferring deep neural networks from simulation to the real world.
\newblock In {\em 2017 IEEE/RSJ international conference on intelligent robots and systems (IROS)}. IEEE, 2017.

\bibitem{sadeghi2016cad2rl}
Fereshteh Sadeghi and Sergey Levine.
\newblock Cad2rl: Real single-image flight without a single real image.
\newblock In {\em Robotics: Science and Systems (RSS)}, 2016.

\bibitem{liu2023zero1to3zeroshotimage3d}
Ruoshi Liu, Rundi Wu, Basile~Van Hoorick, Pavel Tokmakov, Sergey Zakharov, and Carl Vondrick.
\newblock Zero-1-to-3: Zero-shot one image to 3d object, 2023.

\bibitem{sargent2024zeronvszeroshot360degreeview}
Kyle Sargent, Zizhang Li, Tanmay Shah, Charles Herrmann, Hong-Xing Yu, Yunzhi Zhang, Eric~Ryan Chan, Dmitry Lagun, Li~Fei-Fei, Deqing Sun, and Jiajun Wu.
\newblock Zeronvs: Zero-shot 360-degree view synthesis from a single image, 2024.

\bibitem{tian2025viewinvariantpolicylearningzeroshot}
Stephen Tian, Blake Wulfe, Kyle Sargent, Katherine Liu, Sergey Zakharov, Vitor Guizilini, and Jiajun Wu.
\newblock View-invariant policy learning via zero-shot novel view synthesis, 2025.

\bibitem{chen2024roviaugrobotviewpointaugmentation}
Lawrence~Yunliang Chen, Chenfeng Xu, Karthik Dharmarajan, Muhammad~Zubair Irshad, Richard Cheng, Kurt Keutzer, Masayoshi Tomizuka, Quan Vuong, and Ken Goldberg.
\newblock Rovi-aug: Robot and viewpoint augmentation for cross-embodiment robot learning, 2024.

\bibitem{khazatsky2024droid}
Alexander Khazatsky, Karl Pertsch, Suraj Nair, and \emph{et al.}
\newblock {DROID}: A large-scale in-the-wild robot manipulation dataset.
\newblock {\em arXiv preprint arXiv:2403.12945}, 2024.

\bibitem{Gulrajani2021Search}
Ishaan Gulrajani and David Lopez-Paz.
\newblock In search of lost domain generalization.
\newblock In {\em International Conference on Learning Representations (ICLR)}, 2021.

\bibitem{Jayaraman2018LearningViewpointInvariance}
Dinesh Jayaraman and Kristen Grauman.
\newblock Learning viewpoint-invariant visual representations by predicting views from novel viewpoints.
\newblock {\em Proceedings of the IEEE Conference on Computer Vision and Pattern Recognition (CVPR)}, 2018.

\bibitem{Wu2023NeuralNBV}
Xiaoyu Wu and colleagues.
\newblock Neural next-best-view planning for active 3d reconstruction.
\newblock In {\em Proceedings of the IEEE Conference on Computer Vision and Pattern Recognition (CVPR)}, 2023.

\bibitem{Lin2023NeuralImplicitActiveVision}
Ying Lin and colleagues.
\newblock Neural implicit active vision for scene understanding.
\newblock In {\em International Conference on Computer Vision (ICCV)}, 2023.

\bibitem{Dhami2023MAPNBV}
Harnaik Dhami, Vishnu~D. Sharma, and Pratap Tokekar.
\newblock Map-nbv: Multi-agent prediction-guided next-best-view planning for active 3d object reconstruction.
\newblock {\em arXiv preprint arXiv:2307.04004}, 2023.

\bibitem{Dhami2023PredNBV}
Harnaik Dhami, Vishnu~D. Sharma, and Pratap Tokekar.
\newblock Pred-nbv: Prediction-guided next-best-view planning for 3d object reconstruction.
\newblock {\em arXiv preprint arXiv:2307.04004}, 2023.

\bibitem{hou2024learning}
Jingdong Hou, Han Zhang, and Xiaoming Liu.
\newblock Learning to select views for efficient multi-view understanding.
\newblock In {\em IEEE Conference on Computer Vision and Pattern Recognition (CVPR)}, 2024.

\bibitem{liu2024splatraj}
Xinyi Liu, Tianyi Zhang, Matthew Johnson-Roberson, and Weiming Zhi.
\newblock Splatraj: Camera trajectory generation with semantic gaussian splatting.
\newblock {\em arXiv preprint arXiv:2410.06014}, 2024.

\bibitem{Chen2024GenNBV}
Yi-Ting Chen, Cheng-Wei Hsieh, Yu-Ting Hsu, and Yu-Chiang~Frank Wang.
\newblock Gennbv: Generalizable next-best-view policy for active 3d reconstruction.
\newblock In {\em Proceedings of the IEEE Conference on Computer Vision and Pattern Recognition (CVPR)}, 2024.

\bibitem{wright2024robust}
Herbert Wright, Weiming Zhi, Matthew Johnson-Roberson, and Tucker Hermans.
\newblock Robust bayesian scene reconstruction by leveraging retrieval-augmented priors.
\newblock {\em arXiv preprint arXiv:2411.19461}, 2024.

\bibitem{marchant2014bayesian}
Roman Marchant and Fabio Ramos.
\newblock Bayesian optimisation for informative continuous path planning.
\newblock In {\em 2014 IEEE International Conference on Robotics and Automation (ICRA)}. IEEE, 2014.

\bibitem{calandra2017bayesian}
Roberto Calandra.
\newblock Bayesian modeling for optimization and control in robotics.
\newblock 2017.

\bibitem{williams2006gaussian}
Christopher~KI Williams and Carl~Edward Rasmussen.
\newblock {\em Gaussian processes for machine learning}, volume~2.
\newblock MIT press Cambridge, MA, 2006.

\bibitem{senanayake2017bayesian}
Ransalu Senanayake and Fabio Ramos.
\newblock Bayesian hilbert maps for dynamic continuous occupancy mapping.
\newblock In {\em Conference on Robot Learning}, 2017.

\bibitem{senanayake2024role}
Ransalu Senanayake.
\newblock The role of predictive uncertainty and diversity in embodied ai and robot learning.
\newblock {\em arXiv preprint arXiv:2405.03164}, 2024.

\bibitem{wilson1712reparameterization}
JT~Wilson, R~Moriconi, F~Hutter, and MP~Deisenroth.
\newblock The reparameterization trick for acquisition functions. arxiv 2017.
\newblock {\em arXiv preprint arXiv:1712.00424}, 2017.

\bibitem{Srinivas_2012}
Niranjan Srinivas, Andreas Krause, Sham~M. Kakade, and Matthias~W. Seeger.
\newblock Information-theoretic regret bounds for gaussian process optimization in the bandit setting.
\newblock {\em IEEE Transactions on Information Theory}, 58(5), May 2012.

\bibitem{zhu2025robosuitemodularsimulationframework}
Yuke Zhu, Josiah Wong, Ajay Mandlekar, Roberto Martín-Martín, Abhishek Joshi, Kevin Lin, Abhiram Maddukuri, Soroush Nasiriany, and Yifeng Zhu.
\newblock robosuite: A modular simulation framework and benchmark for robot learning, 2025.

\bibitem{robomimic2021}
Ajay Mandlekar, Danfei Xu, Josiah Wong, Soroush Nasiriany, Chen Wang, Rohun Kulkarni, Li~Fei-Fei, Silvio Savarese, Yuke Zhu, and Roberto Mart\'{i}n-Mart\'{i}n.
\newblock What matters in learning from offline human demonstrations for robot manipulation.
\newblock In {\em arXiv preprint arXiv:2108.03298}, 2021.

\bibitem{bain1995framework}
Michael Bain and Claude Sammut.
\newblock A framework for behavioural cloning.
\newblock In {\em Machine intelligence 15}, 1995.

\bibitem{fujimoto2019off}
Scott Fujimoto, David Meger, and Doina Precup.
\newblock Off-policy deep reinforcement learning without exploration.
\newblock In {\em International conference on machine learning}. PMLR, 2019.

\bibitem{chi2024diffusionpolicyvisuomotorpolicy}
Cheng Chi, Zhenjia Xu, Siyuan Feng, Eric Cousineau, Yilun Du, Benjamin Burchfiel, Russ Tedrake, and Shuran Song.
\newblock Diffusion policy: Visuomotor policy learning via action diffusion, 2024.

\bibitem{zhao2023learning}
Tony~Z Zhao, Vikash Kumar, Sergey Levine, and Chelsea Finn.
\newblock Learning fine-grained bimanual manipulation with low-cost hardware.
\newblock {\em arXiv preprint arXiv:2304.13705}, 2023.

\bibitem{paciorek2003nonstationary}
Christopher Paciorek and Mark Schervish.
\newblock Nonstationary covariance functions for gaussian process regression.
\newblock {\em Advances in neural information processing systems (NeurIPS)}, 2003.

\bibitem{senanayake2018automorphing}
Ransalu Senanayake, Anthony Tompkins, and Fabio Ramos.
\newblock Automorphing kernels for nonstationarity in mapping unstructured environments.
\newblock In {\em Conference on Robot Learning (CoRL)}, 2018.

\bibitem{senanayake2016spatio}
Ransalu Senanayake, Lionel Ott, Simon O'Callaghan, and Fabio~T Ramos.
\newblock Spatio-temporal hilbert maps for continuous occupancy representation in dynamic environments.
\newblock In {\em Advances in Neural Information Processing Systems (NeurIPS)}, 2016.

\bibitem{bardou2025optimizing}
Anthony Bardou and Patrick Thiran.
\newblock Optimizing through change: Bounds and recommendations for time-varying bayesian optimization algorithms.
\newblock {\em arXiv preprint arXiv:2501.18963}, 2025.

\end{thebibliography}

\appendix

\subsection{Computing Resources}
We trained and fine-tuned all policies on an RTX 4090 (12 GB VRAM). A single iteration of data collection and policy training/evaluation takes roughly one hour, while each BO step completes in about 30 seconds.

\subsection{Proofs of Theoretical Guarantees}

We formalize the theoretical guarantees for \emph{Vantage} in this section.  
The true success rate function $f(\theta)$ is modeled as a sample from a Gaussian Process prior, and each fine-tuning experiment yields a noisy observation $y_t = f(\theta_t) + \varepsilon_t$.  
The GP posterior mean $\mu_t(\theta)$ and variance $\sigma_t^2(\theta)$ provide estimates of performance and uncertainty, and the GP-UCB acquisition balances exploration and exploitation. Below we establish regret bounds, convergence rates, and robustness under camera placement error.

\begin{theorem}[GP-UCB Regret Bound]\label{thm:gpucb}
Let \(f:\Theta\to\mathbb{R}\) be a mapping from angles to success rates, drawn from a GP prior with kernel \(k\).  At each round \(t=1,\dots,T\) choose, 
\[
\theta_t \;=\;\arg\max_{\theta\in\Theta}\Bigl[\mu_{t-1}(\theta)\;+\;\sqrt{\beta_t}\,\sigma_{t-1}(\theta)\Bigr],
\]
and observe \(y_t=f(\theta_t)+\varepsilon_t\) with \(\varepsilon_t\) zero-mean sub-Gaussian noise.  Then, with probability at least \(1-\delta\),
\[
R(T)\;=\;\sum_{t=1}^T\sum_{i=1}^q\bigl[f(\theta^*)-f(\theta_{t,j})\bigr]
\;=\;
O\!\bigl(\sqrt{qT\,\gamma_{qT}\,\beta_T}\bigr),
\]
where \(\theta^*=\arg\max_\Theta f\), \(\gamma_T\) is the maximum information gain after \(T\) steps, and \(\beta_T\) is chosen as in~\cite{Srinivas_2012}.
\end{theorem}

\begin{proof}
By Theorem 2 of Srinivas \emph{et al.} (2012), with probability \(1-\delta\), for all \(t\) and all \(\theta\),
\[
\bigl|f(\theta)-\mu_{t-1}(\theta)\bigr|
\;\le\;
\sqrt{\beta_t}\,\sigma_{t-1}(\theta).
\]
Hence the instantaneous regret \(r_t = f(\theta^*) - f(\theta_t)\) satisfies
\[
r_t \;\le\;
2\,\sqrt{\beta_t}\,\sigma_{t-1}(\theta_t).
\]
Summing and applying Cauchy–Schwarz with with the definition of information gain gives
\[
\begin{aligned}
R(T)
&= \sum_{t=1}^T r_t \le 2\sum_{t=1}^T \sqrt{\beta_t}\,\sigma_{t-1}(\theta_t) \\
&\le 2\sqrt{\Bigl(\sum_t \beta_t\Bigr)
             \Bigl(\sum_t\sigma^2_{t-1}(\theta_t)\Bigr)}
   = O\!\bigl(\sqrt{T\,\gamma_T\,\beta_T}\bigr)
\end{aligned}
\]
\end{proof}

\begin{theorem}[Average Success Convergence]\label{thm:avgconv}
Under the same setting as Theorem \ref{thm:gpucb}, with probability at least \(1-\delta\),
\[
\frac{1}{T}\sum_{t=1}^T J\bigl(\pi_{\theta_t}\bigr)
\;\ge\;
J\bigl(\pi_{\theta^*}\bigr)
\;-\;
O\!\Bigl(\sqrt{\tfrac{\gamma_T\,\beta_T}{T}}\Bigr).
\]
In particular, the mean success converges to the optimum at rate \(O(T^{-1/2})\).
\end{theorem}

\begin{proof}
From Theorem \ref{thm:gpucb}, with high probability,
\[
\begin{aligned}
\sum_{t=1}^T\bigl[J(\pi_{\theta^*})-J(\pi_{\theta_t})\bigr]
&= O\!\bigl(\sqrt{T\,\gamma_T\,\beta_T}\bigr) \\
\implies J(\pi_{\theta^*})
&- \frac{1}{T}\sum_{t=1}^T J(\pi_{\theta_t})
= O\!\Bigl(\sqrt{\tfrac{\gamma_T\,\beta_T}{T}}\Bigr)
\end{aligned}
\]
which yields the stated bound.
\end{proof}

\begin{theorem}[Robustness under camera placement error]\label{thm:robustness_appendix}
Let $f:\mathbb{R}^d \to \mathbb{R}$ be modeled with a Gaussian process prior with squared--exponential kernel,
\[
k(x,x') = \sigma_f^2 \exp\!\Big(-\tfrac{\|x-x'\|^2}{2\ell^2}\Big).
\]
Suppose Bayesian optimization selects query viewpoints $x_t \in \mathbb{R}^d$, but the executed viewpoints are 
$\tilde{x}_t = x_t + \varepsilon_t$, with $\varepsilon_t \sim \mathcal{N}(0,\sigma_x^2 I_d)$ i.i.d.\ perturbations due to camera placement error. Define the effective objective $g(x) = \mathbb{E}[\,f(x+\varepsilon)\,].$

Then $g$ is governed by a GP with kernel,
\[
k_g(x,x') \;=\; \sigma_f^2
\Big(\tfrac{\ell^2}{\ell^2+2\sigma_x^2}\Big)^{\!d/2}
\exp\!\Big(-\tfrac{\|x-x'\|^2}{2(\ell^2+2\sigma_x^2)}\Big),
\]
which is a squared-exponential kernel with effective hyper-parameters,
$\ell_{\!\mathrm{eff}}^{\,2}=\ell^2+2\sigma_x^2,
\qquad
\sigma_{f,\mathrm{eff}}^{\,2}=\sigma_f^2\Big(\tfrac{\ell^2}{\ell^2+2\sigma_x^2}\Big)^{\!d/2}.$
\end{theorem}

\begin{proof}
Linearity of expectation preserves Gaussianity, so $g$ is again a Gaussian process. Its covariance is.
\[
k_g(x,x') \;=\; \mathbb{E}_{\varepsilon,\varepsilon'}\big[k(x+\varepsilon,\,x'+\varepsilon')\big],
\]
where $\varepsilon,\varepsilon'\!\sim\!\mathcal{N}(0,\sigma_x^2I_d)$ are independent.  
Let $\mu = x-x'$ and define $\delta = \varepsilon-\varepsilon'$. Then, $\delta \sim \mathcal{N}(0,2\sigma_x^2I_d)$, and hence,
\[
k_g(x,x') \;=\; \sigma_f^2 \,\mathbb{E}_{\delta}\!\left[
\exp\!\Big(-\tfrac{\|\mu+\delta\|^2}{2\ell^2}\Big)\right].
\]
This expectation can be evaluated using the Gaussian moment identity, $\mathbb{E}_{\delta\sim\mathcal{N}(0,\Sigma)}\!
\Big[\exp\!\Big(-\tfrac{1}{2}(\mu+\delta)^\top A(\mu+\delta)\Big)\Big]$
\[
\begin{aligned}
 = |I+\Sigma A|^{-1/2}\,
   \exp\!\Big(-\tfrac{1}{2}\mu^\top A(I+\Sigma A)^{-1}\mu\Big).
\end{aligned}
\]

Substitute $A=\tfrac{1}{\ell^2}I_d$ and $\Sigma=2\sigma_x^2 I_d$. Then,
\[
|I+\Sigma A|^{-1/2} 
=\Big(\tfrac{\ell^2}{\ell^2+2\sigma_x^2}\Big)^{\!d/2},
A(I+\Sigma A)^{-1}
=\tfrac{1}{\ell^2+2\sigma_x^2}I_d.
\]
\[
\implies k_g(x,x')=\sigma_f^2
\Big(\tfrac{\ell^2}{\ell^2+2\sigma_x^2}\Big)^{\!d/2}
\exp\!\Big(-\tfrac{\|\mu\|^2}{2(\ell^2+2\sigma_x^2)}\Big)
\]

\end{proof}

\subsection{Defining the space of allowed camera placements}

Let \(b \in \mathbb{R}^3\) denote the robot's base position and let \(r > 0\) be the fixed radial distance at which the camera is placed.  Let the sphere of radius \(r\) around \(b\) is defined by $S \;=\; \bigl\{\,x \in \mathbb{R}^3 : \|x - b\| = r \bigr\}.$ For experimental purposes, we restrict our space of viewpoints to the \emph{spherical quadrant} in front of the robot and above the table (the blue area in Fig.~\ref{fig:intro}):
\[
\Theta \;=\; \bigl\{\,x \in S : 
\;(x - b)\cdot u \;\ge\; 0
\quad\text{and}\quad
(x - b)\cdot n \;\ge\; 0
\bigr\},
\]
where \(u\) is the unit vector pointing forward from the robot’s base and \(n\) is the upward normal to the table plane. Any viewpoint \(\theta \in \Theta\) can be uniquely parameterized by the horizontal and vertical angles \(\theta_h\) and \(\theta_v\), respectively:
\[
x(\theta_h,\theta_v)
= b + r
\begin{bmatrix}
\cos\theta_v \,\cos\theta_h \\[6pt]
\cos\theta_v \,\sin\theta_h \\[6pt]
\sin\theta_v
\end{bmatrix},
\quad
\]
\[
\theta_h \in \bigl[-\tfrac{\pi}{2}, \tfrac{\pi}{2}\bigr],
\quad
\theta_v \in \bigl[-\tfrac{\pi}{4}, \tfrac{\pi}{4}\bigr]
\]
Any $\theta = (\theta_h,\theta_v) \in \Theta$ is considered a valid camera placement.
Robot policies are initially trained from the viewpoint \(\theta = (0,0)\). The resulting model is then fine-tuned using data collected from one additional viewpoint. To quantify the variability of our method, in all experiments, this process is repeated four times, each with a different randomly selected initialization point for the GP. Next, we evaluate the model's performance on a uniform grid over the viewing space \(\Theta\) and fit a GP surrogate to these measurements. By applying the q-UCB acquisition function to the GP, we identify the next eight most informative viewpoints for further data collection and fine-tuning.

We parameterize each viewpoint by normalized coordinates \(\nu_h,\nu_v \in [0,1]\). The GP is trained on these normalized inputs and thus only \textit{sees} values in the unit square. To convert a normalized sample \((\nu_h,\nu_v)\) into actual camera angles, we apply the affine transformation:
\[
\theta_{h/v}
= \bigl(\nu_{h/v} - 0.5\bigr)\,\bigl(\theta_{h/v}^{\max} - \theta_{h/v}^{\min}\bigr)
+ \frac{\theta_{h/v}^{\max} + \theta_{h/v}^{\min}}{2}\,
\]

All methods (grid search, random search, and \emph{Vantage}) were allocated identical compute resources. Each task–policy combination underwent eight fine-tuning steps per iteration per method, and 4 iterations were done leading to 32 models fine-tuned for each combination.

\begin{figure*}[t]
  \centering

  \begin{subfigure}[t]{0.75\textwidth}
    \centering
    \foreach \n in {1,...,3}{%
      \includegraphics[width=0.32\linewidth]{images/gif_plot/Lift_BC_\n.pdf}\hfill%
    }\\[2pt]
    \foreach \n in {1,...,3}{%
      \includegraphics[width=0.32\linewidth]{images/gif_plot/Square_BC_\n.pdf}\hfill%
    }\\[2pt]
    \foreach \n in {1,...,3}{%
      \includegraphics[width=0.32\linewidth]{images/gif_plot/PickPlace_BC_\n.pdf}\hfill%
    }
    \label{fig:bc_grids}
  \end{subfigure}


  \caption{Rollout progression across camera viewpoints and training iterations for each task. Each row corresponds to a different manipulation task, namely \emph{Lift}, \emph{Square} and \emph{PickPlace}, while columns illustrate sequential outputs of the BC policy. The plots visualize how task performance evolves as the policy is fine-tuned with optimized viewpoints, bigger circles indicate points from the current iteration while smaller ones indicate older iterations.}
  \label{fig:bc_combined}
\end{figure*}

\subsection{Discussion on the BO Formulation}
Let the dataset collected from the default viewpoint be $D_0$ and the candidate datasets (viewpoints) suggested by BO at each iteration be  $D_1,D_2,\dots,D_n$. Let the base policy trained on $D_0$ be $\pi_0$. At each iteration $i$, we form the union $D_0 \cup D_{i}$ and fine-tune the base policy using this data. At every iteration we restart from $\pi_0$ and fine-tune it on $D_0 \cup D_{i}$, rather than continuing from the previously fine-tuned policy or considering the cumulative. BO is optimizing the performance of $\pi_0$ when fine-tuned on $D_0 \cup D_{i}$: 
\[
f(D_i) = R\bigl(\text{FineTune}(\pi_0; D_0 \cup D_i)\bigr),
\]
where $R$ is the success rate. If we were to fine-tune the policy cumulatively, the objective function that BO tries to optimize will become non-stationary. In other words, the mapping becomes,
\[
f_i(D_i) = R\bigl(\text{FineTune}(\pi_{i-1}; D_i)\bigr),
\]
and \( f_i \) depends on the entire training history up to iteration \( i-1 \).
This makes the objective dynamic and path-dependent, violating one of the key assumptions of standard BO, that it is optimizing a static, smooth, and approximately stationary black-box function. To obtain better results than ours, one might consider a BO formulation using nonstationary Gaussian process approximations~\cite{paciorek2003nonstationary,senanayake2018automorphing,senanayake2016spatio} or time-dependent acquisition functions~\cite{bardou2025optimizing}. 

\subsection{Additional Results}

See Figures~\ref{fig:bc_combined} and~\ref{fig:hyperplane_all} for Vantage's convergence and success rates, respectively.
\begin{figure}[ht]
  \centering
  \includegraphics[width=0.4\textwidth]{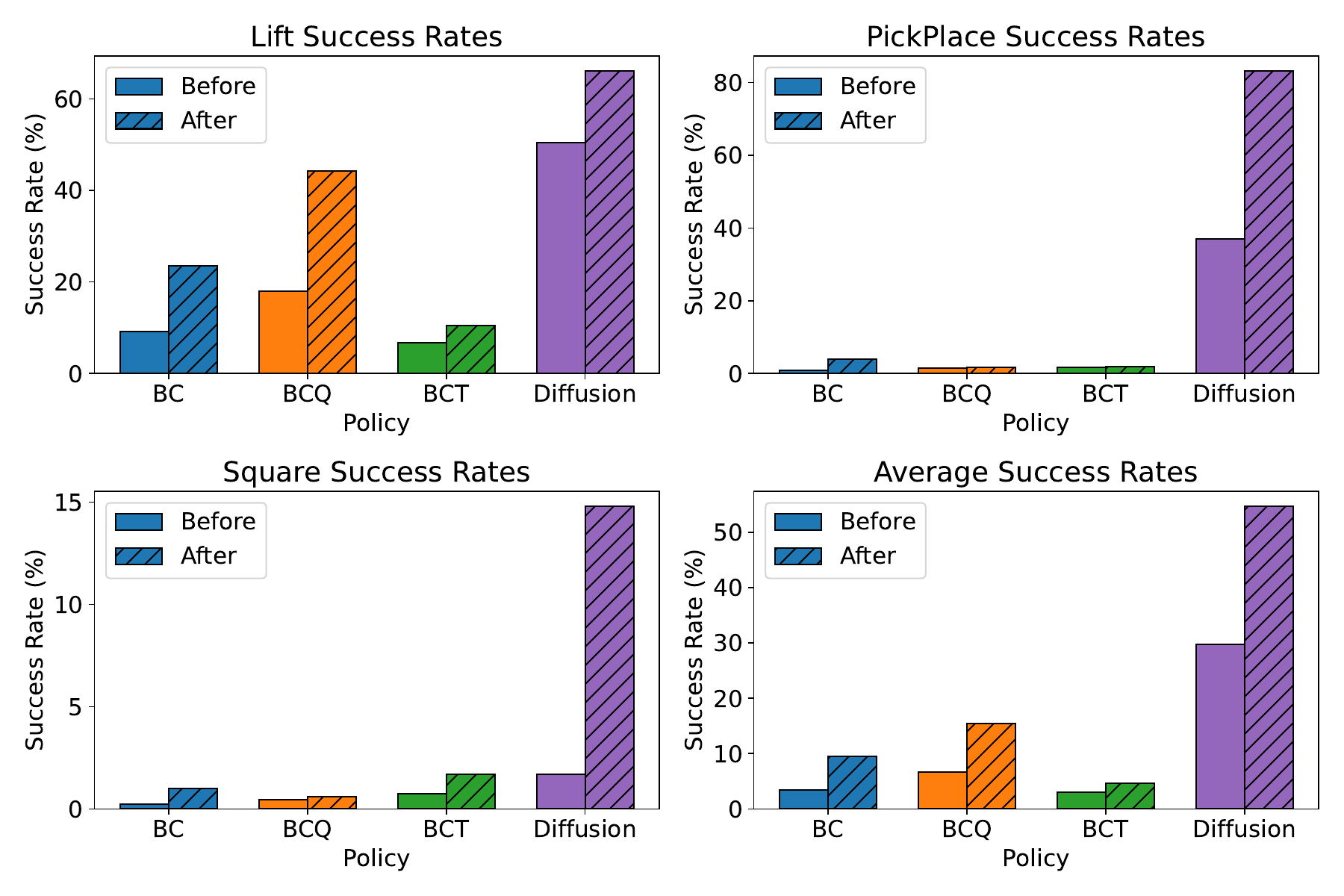}
  \caption{Success rates for each model across Lift, PickPlace, Square, and Average metrics. Solid bars indicate default model and hatched bars indicate after applying Vantage}
  \label{fig:hyperplane_all}
\end{figure}

\end{document}